\documentclass[runningheads]{llncs}
\usepackage{graphicx}
\usepackage{comment}
\usepackage{amsmath,amssymb} 
\usepackage{color}
\usepackage{amsmath}
\usepackage{amssymb}
\usepackage{subfigure}
\usepackage{booktabs} 
\usepackage{graphicx}

\newsavebox\CBox
\def\textBF#1{\sbox\CBox{#1}\resizebox{\wd\CBox}{\ht\CBox}{\textbf{#1}}}

\begin{document}
\pagestyle{headings}
\mainmatter

\title{Learning latent representations across multiple data domains using Lifelong VAEGAN} 

\titlerunning{Accepted as a conference paper at ECCV 2020} 
\authorrunning{Accepted as a conference paper at ECCV 2020} 
\author{Fei Ye and Adrian G. Bors}
\institute{Department of Computer Science, University of York, York YO10 5GH, UK}
\setlength{\textfloatsep}{2pt}
\setlength{\abovedisplayskip}{3pt} 
\setlength{\belowdisplayskip}{3pt}
\setlength{\abovecaptionskip}{0cm}

\maketitle

\begin{abstract}
The problem of catastrophic forgetting occurs in deep learning models trained on multiple databases in a sequential manner. Recently, generative replay mechanisms (GRM), have been proposed to reproduce previously learned knowledge aiming to reduce the forgetting. However, such approaches lack an appropriate inference model and therefore can not provide latent representations of data. In this paper, we propose a novel lifelong learning approach, namely the Lifelong VAEGAN (L-VAEGAN), which not only induces a powerful generative replay network but also learns
meaningful latent representations, benefiting representation learning. L-VAEGAN can allow to automatically embed the information associated with different domains into several clusters in the latent space, while also capturing semantically meaningful shared latent variables, across different data domains. The proposed
model supports many downstream tasks that traditional generative replay methods can not, including interpolation and inference across different
data domains.

\end{abstract}

\section{Introduction}

The lifelong learning framework describes an intelligent learning process capable of remembering all previously learned knowledge from multiple sources, such as different databases \cite{LifeLong_review}. The ability of continuous, or lifelong learning, is an inherent characteristic of humans and animals, which helps them to adapt to the environment during their entire life. However, such characteristics remain an open challenge for deep learning models. The current state-of-the-art deep learning approaches perform well on many individual databases \cite{ScalObjDet,R-CNN}, but suffer from catastrophic forgetting when attempting to learn data associated with new tasks \cite{Expertgate,Lifelong_sentiment,EvidenceLifelong,EncoderLifelong,Lifelong_hierarchical,Lifelong_expandable}. For example when a deep neural network is trained on a new database, its parameters are updated in order to learn new information while their previous values are lost. Consequently, their performance on the previously learnt tasks degenerates.

In order to alleviate the catastrophic forgetting problem, memory-based approaches use a buffer to store a small subset of previously seen data samples \cite{OnlineContious,GradientLifelong,TinyLifelong}. However, such approaches cannot be seen as lifelong learning models and they do not scale well when increasing the number of databases defining different tasks. Shin {\em et al.} \cite{Generative_replay} proposed a learning model employing the Generative Replay Mechanism (GRM). The idea of the GRM is to train a generative replay network to reproduce previously learnt knowledge by using adversarial learning. A classifier is then trained using jointly generative replay data and data sampled from the current database. This approach was only applied on prediction tasks. Although recent studies, such as the one from \cite{MemoryReplayGAN}, have been used to generate images from new classes without forgetting, generative replay approaches do not learn the representation of data and therefore can not be extended to be used in a broad range of tasks. Learning meaningful and disentangled representations of data was shown to benefit many tasks, but they have not been explored so far within the lifelong learning methodology \cite{Lifelong_VAE,LifelongUnsupervisedVAE}. In this paper, we propose a new lifelong learning model which not only learns a GRM but also induces accurate inference models, benefiting on representation learning. 

This research study brings the following contributions~: 
\begin{itemize}
	\item [1)] In order to address the drawbacks of generative replay approaches, we propose a novel lifelong learning model aiming to learn informative latent variables over time.	
	\item [2)] We show that the proposed lifelong learning model can be extended for unsupervised, semi-supervised and supervised learning with few modifications.
    \item[3)] We propose a two-step optimization algorithm to train the proposed model. The latent representation learned by the proposed model can capture both task-specific generative factors and semantic meaningful shared latent variables across different domains over time.
    \item [4)] We provide a theoretical insight into how GRM models are used for lifelong learning in artificial systems. 
\end{itemize}

\section{Related works}
\label{RelWork}

The lifelong learning was approached in previous research studies from three different perspectives: by using regularization, dynamic architectures, and by employing memory replay. Regularization approaches in order to alleviate catastrophic forgetting add an auxiliary term that penalizes changes in the weights when the model is trained on a new task  \cite{BoostingTransfer,Distilling_nets,LessForgetting,EWC,Lwf,LearnAdd,LifeLong_combination,Generative_replay}. Dynamic architectures would increase the number of neurons and network layers in order to adapt to learning new information \cite{PNN}. Most memory replay approaches are using generative models, such as either Generative Adversarial Networks (GANs) \cite{GAN} or Variational Autoencoders (VAEs) \cite{VAE}) to replay the previously learnt knowledge. For instance, Wu {\em et al.} \cite{MemoryReplayGAN} proposed a novel lifelong generative framework, namely Memory Replay GANs (MeRGANs), which mainly generates images from new categories under the lifelong learning setting. The Lifelong GAN \cite{LGAN_conditional} employs image to image translation. However, both models from \cite{MemoryReplayGAN,LGAN_conditional} lack an image inference procedure and Lifelong GAN would need to load all previously learnt data for the generation task. Approaches employing both generative and inference mechanisms are based on the VAE framework \cite{Lifelong_VAE,GenerativeLifelong}. However, these approaches have degenerating performance when learning high-dimensional data, due to lacking a powerful generator. 

Hybrid VAE-GAN methods learn an inference model from a GAN model, which can also capture data representations, which is specific to the VAE. Adversarial learning is performed in order to match either the data distribution \cite{HybridGAN}, the latent variables distribution \cite{AAE}, or their joint distributions \cite{VAE_symmetric,AFL,AdLearnInf,Alice,AdbVB,ASymVAE,Veegan}. These methods perform well only when trained on a single dataset and their performance would degenerate when learning a new task. 

This paper is the first research study to propose a novel hybrid lifelong learning model, which not only addresses the drawback of the existing hybrid methods but also provides inference mechanisms for the GRM, benefiting on many downstream tasks across domains under the lifelong learning framework. The approach proposed in this paper also addresses disentangled representation learning \cite{VI_disentangled} in the context of lifelong learning. Many recent approaches would aim to modify the VAE framework in order to learn a meaningful representation of the data by imposing a large penalty on the Kullback-Leibler (KL) divergence term \cite{UnVAE,JVAE,baeVAE} or on the total correlation  latent variables \cite{IsolatingVAE,VAETCE,DC_Disentanglement,DisentanglingByFactorising}. These approaches perform well on independent and identically distributed data samples from a single domain. However, they are unable to learn the information from piecewise changing stationary data from multiple databases, because they suffer from catastrophic forgetting. 

\section{The Lifelong VAEGAN}
\label{LifeL-VAEGAN}

In this section, we introduce the optimization algorithm used for training the proposed model when learning several databases without forgetting.

\subsection{Problem formulation}

The lifelong learning problem consists of learning a sequence of of $K$ tasks, each characterized by a distinct database, corresponding to the data distributions $p({\bf x}^1), p({\bf x}^2), \ldots, p({\bf x}^K)$. During the $k$-th database learning, we only access the images sampled from $p({\bf x}^k)$. Most existing lifelong learning approaches focus on prediction or regression tasks. Meanwhile, in this research study we focus on modelling the overall data distribution $p ({\bf x})$, by learning latent representations over time:
\begin{equation}
\begin{aligned}
p ({\bf x}): = \int \prod\limits_{i = 1}^K p ({\bf x}^i|{\bf z}) p({\bf z}) d{\bf z}
\end{aligned}
\end{equation}
where ${\bf z}$ represents the latent variables defining the information of all previously learnt databases. Besides addressing unsupervised learning, we also incorporate discrete variables into the optimization path in order to capture category discriminating information. Moreover, we also consider the semi-supervised learning problem where we consider that in each dataset we only have some labelled data while the rest are unlabeled.

\subsection{Data generation from prior distributions}

In the following we aim to learn two separate latent representations for capturing discrete and continuous variations of data. The discrete data are denoted as ${\bf c} = \{ {\bf c}_i | i=1, \ldots, L \}$ where $L$ is the dimension of the discrete variable space, while the continuous variables ${\bf z}$ are sampled from a normal distribution ${\cal N} ( {0,{\mathop{\rm I}\nolimits} })$. We also consider the domain variable ${\bf a} = \{{\bf a}_j | j=1,\ldots,K\}$, defining each database and aiming to capture the information characterizing its task. The generation process considering the three latent variables ${\bf c}$, ${\bf z}$ and ${\bf a}$ is defined as:
\begin{equation}
\begin{aligned}
&{\bf c} \sim {\rm{Cat}}\left( {K = L,p = 1/K} \right),\bf{z} \sim {\cal N} ( {0,{\mathop{\rm I}\nolimits} }),\\& {\bf a} \sim {\rm{Cat}}\left( p_1,p_2,\ldots,p_K \right), {\bf x} \sim p_{\theta} ({\bf x} | {\bf z},{\bf a},{\bf c}),
\end{aligned}
\end{equation}
where $\rm{Cat}(\cdot)$ is the Categorical distribution, and $p_{\bf{\theta}}({\bf x} | {\bf z},{\bf a},{\bf c})$ is the distribution characterizing the generator implemented by a neural network with trainable parameters, $\theta$.
By incorporating the domain variables ${\bf a}$ in the inference model helps to generate images characteristic to a specific task. We consider the Wasserstein GAN (WGAN) \cite{WGAN} loss with the gradient penalty \cite{ImproveWGAN}, which is defined by~:
\begin{equation}
\begin{aligned}
 \mathop {\min }\limits_G \mathop {\max }\limits_D \mathcal L_{GAN}^G (\theta ,\omega ) = &
\mathbb{E}_{{\bf z} \sim p ( {\bf z}),
{\bf c} \sim p ( {\bf c} ), {\bf a} \sim p ( {\bf a} )} [ D ( G ( {\bf c},{\bf z},{\bf a}) ) ] - \\
& \mathbb{E}_{p ( {\bf x} )} [ D ( {\bf x}) ] + \lambda \mathbb{E}_{\tilde{\bf x} \sim p (\tilde{\bf x})} [ ( \| \nabla_{\tilde{\bf x}} D ( \tilde{\bf x})\|_2 - 1 )^2 ]
\end{aligned}
\label{Generate}
\end{equation}
 where we introduce a Discriminator $D$, defined by the trainable parameters $\omega$, $p({\bf x})$ denotes the true data distribution, and the third term is the gradient weighted by the penalty $\lambda$. The adversarial loss allows the Generator and Discriminator to be trained alternately such that the Discriminator aims to distinguish real from generated data, while the Generator tends to fool the Discriminator through aiming to generate realistic data \cite{WGAN,GAN}. 

\subsection{Training the inference model}

Most GAN-based lifelong methods  \cite{Generative_replay,MemoryReplayGAN,LifelongGAN} do not learn an accurate inference model and therefore can not derive a meaningful data representation. For the model proposed in this paper, we consider three differentiable non-linear functions $f_\varsigma(\cdot)$, $f_\varepsilon(\cdot)$, $f_\delta(\cdot)$, aiming to infer three different types of latent variables $\{ {\bf z},{\bf c},{\bf a} \}$. We implement $f_\varsigma(\cdot)$ by using the Gaussian distribution ${\cal N} (\mu ,\sigma )$ where $\mu  = \mu_\varsigma ( {\bf x})$ and $\sigma  = 
\sigma_\varsigma ( {\bf x} )$ are given by the outputs of a neural network with trainable parameters $\varsigma$. We use the reparameterization trick \cite{VAE,VAE2} for sampling 
${\bf z} = \mu  + \pi  \otimes \sigma $, where $\pi$ is a random noise vector sampled from ${\cal N} ( 0,\mathop{\rm I}\nolimits )$, in order to ensure end-to-end training.

\noindent \textbf{Discrete variables.} We can not sample the discrete latent variables ${\bf a}$ and ${\bf c}$ from $f_\varepsilon(\cdot)$ and $f_\delta(\cdot)$, respectively, because the categorical representations are non-differentiable. In order to mitigate this, we use the Gumbel-Max trick \cite{GumbleTrick,GumbleTrick2} for achieving the differentiable relaxation of discrete random variables. The Gumbel-softmax trick was also used in \cite{JVAE,Gumble_softmax,concrete,KDGAN} and its capability of reducing the variation of gradients was shown in \cite{KDGAN}. 

The sampling process of discrete latent variables is defined as:
\begin{equation}
\begin{aligned}
{\bf a}_j = \frac{\exp ( ( \log {\bf a}'_j + {\bf g}_j )/ T )}{\sum\limits_{i=1}^K \exp (( \log {{\bf a'}_i} + {\bf g}_i )/ T )}
\end{aligned}
\end{equation}
 where ${\bf a'}_i$ is the $i$-th entry of the probability defined by the softmax layer characterizing $f_\varepsilon(\cdot)$ and ${\bf a}_j$ is the continuous relaxation of the domain variable, while ${\bf g}_k$ is sampled from the distribution $\mathop{\rm Gumbel}\nolimits (0,1)$ and $T$ is the temperature parameter that controls the degree of smoothness. We use the Gumbel softmax trick for sampling both the domain ${\bf a}$ and the discrete ${\bf c}$ variables. 

\noindent \textbf{The log-likelihood objective function.} GANs lack an inference mechanism, preventing them to capture data representations properly. In this paper we propose to maximize the sample log-likelihood for learning the inference models, defined by $ p({\bf x}) = \int \int \int p({\bf x}|{\bf z},{\bf a},{\bf c}) p({\bf z},{\bf a},{\bf c})\:d{\bf z}\:d{\bf a}\:d{\bf c}$, which is intractable in practice. We therefore derive the following lower bound on the log-likelihood, which is characteristic to VAEs, by introducing variational distributions:
\begin{equation}
\begin{aligned}
\mathcal{L}_{\rm{VAE}}({\bf \theta},{\bf \varsigma },{\bf \varepsilon},{\bf \delta })  = & \mathbb{E}_{q_{\varsigma ,\varepsilon ,\delta }}({\bf z},{\bf a},{\bf c}|{\bf x}) \log [p_\theta ({\bf x}|{\bf z},{\bf a},{\bf c})] - D_{KL}[q_\varsigma ({\bf z}|{\bf x})||p({\bf z})] \\&- \mathbb{E}_{q_\varsigma ({\bf z}|{\bf x})}D_{KL}[q_\varepsilon({\bf a}|{\bf x})||p({\bf a}|{\bf z})] - {D_{KL}}[{q_\delta }({\bf c}|{\bf x})||p({\bf c})]
\end{aligned}
\label{L_VAE}
\end{equation}
where $q_\varsigma ({\bf z} |{\bf x})$, ${{q_{\varepsilon }}({\bf{a}}|{\bf{z}})}$, $q_\delta ({\bf c}|{\bf x})$ are variational distributions modelled by $f_\varsigma(\cdot)$, $f_\varepsilon(\cdot)$, $f_\delta(\cdot)$, respectively. For the third term from (\ref{L_VAE}), we sample from the empirical distribution and then sample ${\bf z}$ from $q_\delta ({\bf c}|{\bf x})$. $p({{\bf{a}}|{\bf z}})$ is the prior distribution ${\rm Cat}(p_1,\ldots,p_K)$, where $p_i$ denotes the probability of the sample belonging to the $i$-th domain. We consider $q_\varepsilon({\bf a}|{\bf z})$ as the task-inference model which aims to infer the task ID for the given data samples. 

\begin{figure}[htbp]
	\centering
		\setlength{\abovecaptionskip}{0pt}   
    \setlength{\belowcaptionskip}{-10pt}
		\includegraphics[scale=0.54]{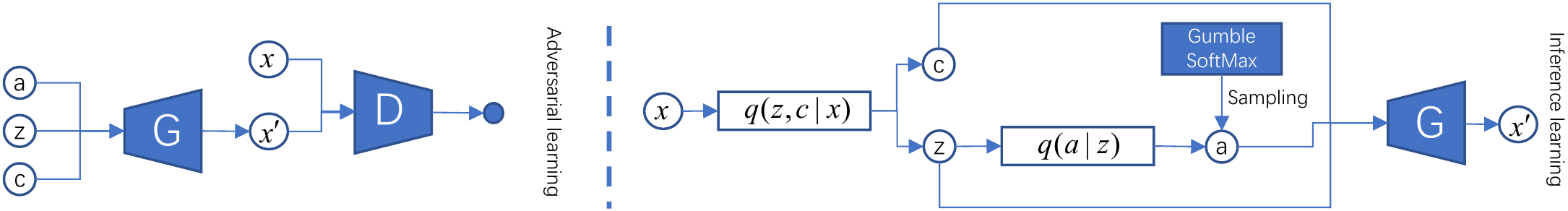}  \\
		(a)  Generator network  \qquad \qquad \qquad \qquad \qquad (b) Inference network \\
	\caption{The graph structure for the proposed Lifelong VAEGAN (L-VAEGAN) model, where G and D denote the Generator and Discriminator, respectively. }
	\label{Figure1}
\end{figure}

For the supervised learning setting, the auxiliary information such as the class labels can be used to guide the inference model. We minimise the cross-entropy loss $\eta (\cdot,\cdot)$ for $q_\delta ({\bf c}|{\bf x})$ and $q_\varepsilon ( {\bf a}|{\bf z})$ as:
\begin{equation}
\begin{aligned}
{\mathcal L}_{\bf a} (\varepsilon) = \mathbb{E}_{({\bf x},{\bf a}^*) \sim ( {\bf X},{\bf A}),{\bf z} \sim {q_\varsigma }({\bf{z}}|{\bf{x}})} \eta ( q_{\varepsilon} ({\bf a}|{\bf z}),{\bf a}^* )
\end{aligned}
\end{equation}
\begin{equation}
\begin{aligned}
{\mathcal L}_{\bf c}(\delta ) = \mathbb{E}_{{( {\bf x},{\bf y}} ) \sim ( {\bf X},{\bf Y} )} \eta ( q_{\delta} ( {\bf c}|{\bf x}),{\bf y} )
\end{aligned}
\end{equation}
where ${\bf X}$ and ${\bf Y}$ represent the empirical data and target distributions, respectively.  ${\bf a}^*$ is the variable drawn from ${\bf A}$ which represents the Categorical distribution ${\rm{Cat}}(m_1,\ldots,m_k)$, where $m_i$ is the probability of seeing $i$-th task, characterizing the corresponding database. The graph structure of the Generator and Inference networks of the proposed Lifelong VAEGAN (L-VAEGAN) is shown in Figs.~\ref{Figure1}-a and \ref{Figure1}-b, respectively, where the variable ${\bf a}$ is conditioned on ${\bf z}$. The proposed model is flexible to be extended for recognizing new tasks by automatically appending the domain variable ${\bf a}$ and optimizing the task-inference model $q_{\varepsilon }({\bf a}|{\bf z})$ when faced with learning a new task.

\vspace{-10pt}
\section{Theoretical analysis of the GRM}
\label{TheoGenReplay}
\vspace{-5pt}

In this section, we analyze the GRM used in lifelong learning. The proofs are provided in Appendix A-G (https://github.com/dtuzi123/LifelongVAEGAN).

\begin{definition}
We define the distribution modelling the lifelong learned data as $p(\tilde{\bf x}^t) $, which is encoded through $G_{\gamma_t}({\bf z,c,d})$. The assumption is that the network has learnt the information from all given databases $t=1,\ldots,K$, and this information is stored, refined and processed across various tasks, where $\gamma_t$ is the generator parameter updated after the $t$-th task learning. 
\end{definition}

\begin{definition}
Let us define 
\begin{equation}
p({\tilde{\bf x}^t}|\tilde{\bf x}^{t-1},{\bf x}^t) = \exp ( - \Gamma (p(\tilde{\bf x}^{t-1},{\bf x}^t),p(\tilde{\bf x}^t))) 
\label{Def2}
\end{equation}
as the probability of generated data $\tilde{\bf x}^t$ when observing $\tilde{\bf x}^{t-1}$ and ${\bf x}^t$, where $\Gamma(\cdot)$ is a probabilistic measure of comparison between two distributions, 
which can be the {\cal f}-divergence \cite{f-GAN}, or the Wasserstein distance \cite{WGAN} (Earth-mover distance).
\end{definition}

\begin{theorem}
By marginalizing over $\tilde{\bf x}^{t-1}$ and ${\bf x}^t$, on $p(\tilde{\bf x}^t| \tilde{\bf x}^{t-1},{\bf x}^t)$, the resulting marginal distribution $p(\tilde{\bf x}^t)$ encodes the statistical correlations from all previously learnt distributions.
\end{theorem}

\begin{proof}
By using mathematical induction over the lifelong learning of the probabilities associated with various tasks, the marginal distribution is rewritten as:

\begin{equation}
\begin{aligned}
p(\tilde{\bf x}^t) = & \int  \ldots \int p(\tilde{\bf x}^1) \prod\limits_{i = 0}^{t-2} p (\tilde{\bf x}^{t - i} | \tilde{\bf x}^{t-i-1},{\bf x}^{t-i})  \prod\limits_{i = 0}^{t-2} p({\bf x}^{t-i})d \tilde{\bf x}^1 \ldots 
d \tilde{\bf x}^{t - 1} d{\bf x}^2 \ldots d{\bf x}^t
\end{aligned}
\end{equation}
\end{proof} 

\begin{lemma}
The data probability $p(\tilde{\bf x}^t)$ approximates the true joint distribution $\prod\nolimits_{i = 1}^t {p({\bf x}^i)}$ when all previously learnt distributions are the exact approximations to their target distributions while learning every given task.
\end{lemma}

In the following we extend the theoretical analysis on the domain adaptation problem from \cite{Theoretical_lifelong} (Theorem 2) in order to analyze how the knowledge learned by GRMs is lost during the lifelong learning.
\begin{theorem}
Let us consider two vector samples, one corresponding to the generated data $\{ { \nu}_{t'} \in {\rm{R}^s} | { \nu}_{t'} \sim p({\bf \tilde x}^t)\}$ and another corresponding to the real data $\{  {  \nu}_{t} \in {\rm{R}^s} | {  \nu}_{t} \sim p({\bf{x}}^t) \}$ of sizes $n_t$ and $n_{t'}$. Then let $h^t(\cdot)$ be a new learned model trained on ${\nu _{t'}}$. For any $s' > s$ and $a' < \sqrt 2 $, there is a constant $n_0$ depending on $s'$ satisfying that for any $\delta > 0$ and $\min({\nu_t},{n_{t'}}) \ge n_0\max({\delta ^{ - (s' + 2)}},1)$. Then with the probability of at least $1 - \delta $ for all $h^t$, we have:
\begin{equation}
\begin{aligned}
& E\left( {{h^t}} ({\nu _t}) \right) \le E\left( h^t({\nu_{t'}}) \right) + W\left( {{\nu _t},{\nu _{t'}}} \right) + \sqrt {2\log \left( {\frac{1}{\delta }} \right)/a'} \left( {\sqrt {\frac{1}{{{n_t}}}} {\rm{ + }}\sqrt {\frac{1}{{{n_{t'}}}}} } \right){\rm{ + }}D
\end{aligned}
\end{equation}

\noindent where $E( {{h^t}} ({\nu_t})),E( h^t({\nu_{t'}}))$ denote the observed risk for $\nu_t$ and $\nu_{t'}$, respectively, and $W ( \nu_t,\nu_{t'})$ is the Wassenstein distance between $\nu_t$ and $\nu_{t'}$. $D$ is the combined error when we find the optimal model ${h^t}^\prime  =: \arg \min_{h^t \in {\cal H}} (E( {{h^t}} ({\nu _t}))+E( h^t({\nu_{t'}}))$. 
\end{theorem}

This theorem demonstrates that the performance of a model $h^t$ degenerates on the empirical data distribution $p({\bf x}^t)$. From Theorem 2, we conclude that the lifelong learning becomes a special domain adaptation problem in which the target and source domain are empirical data distributions from the current task and the approximation distribution $G_{\gamma_t}({\bf z,c,d})$. 

\begin{lemma}
From Theorem 2 we have a bound on the accumulated errors across tasks during the lifelong learning.
\begin{equation}
\begin{aligned}
&\sum\nolimits_{i = 1}^K E\left( h^K(\nu_i) \right) \le  \sum\nolimits_{i = 1}^K E\left( h^K (\nu_{i^{(K)}}) \right) +\\& {\bf{W}}\left( \nu_i,\nu_{i^{(K)}} \right)  + \sqrt {2\log \left( {\frac{1}{\delta }} \right)/a'} \left( {\sqrt {\frac{1}{{{n_i}}}}  + \sqrt {\frac{1}{{{n_{i^{(K)}}}}}} } \right) +D_{(i^{(K-1)},i^{(K)})} ,
\end{aligned}
\end{equation}
where $E\left( h^K (\nu_{i^{(K)}}) \right)$ denotes the observed risks on the probability measure ${{\nu _{{i^{(K)}}}}}$ formed by samples drawn from $p({{{\bf{\tilde x}}}^i})$, after they have been learned across K tasks. $D_{(i^{(K-1)},i^{(K)})}$ is the combined error of an optimal model\\ 
\begin{equation}
h^* = \arg \min_{h \in {\cal H}} (E\left( h^K (\nu_{i^{(K-1)}}) \right)+E\left( h^K (\nu_{i^{(K)}}) \right))
\end{equation}
\end{lemma}

\begin{theorem}
By having a learning system, acquiring the information from the given databases, we can define $\log p_{\theta}({\bf x}^{t},\tilde{\bf x}^{t-1})$ as the joint model log-likelihood and $p_{\theta}({\bf z}^{t}, 
{\bf z}^{t-1}|{\bf x}^{t},\tilde{\bf x}^{t-1})$ as the posterior. $\log p_{\theta}({\bf x}^{t},\tilde{\bf x}^{t-1})$ can be optimized by maximizing a lower bound.
\end{theorem}  

\begin{proof}
In this case, we consider two underlying generative factors (latent variables) ${\bf z}^{t}$, ${\bf z}^{t-1}$ for observing the real data ${\bf x}^t$ and generated data $\tilde{\bf x}^{t-1}$, respectively. We then define the latent variable model $p_{\theta}({\bf x}^t,\tilde{\bf x}^{t-1},{\bf z}^t,{\bf z}^{t-1}) = p_{\theta}({\bf x}^t,\tilde{\bf x}^t|{\bf z}^t,{\bf z}^{t-1})p({\bf z}^{t})p({\bf z}^{t-1})$ and its marginal log-likelihood is approximated by a lower bound~:
\begin{equation}
\begin{aligned}
\log p_{\theta}({\bf x}^{t},\tilde{\bf x}^{t-1}) &
\ge \mathbb{E}_{q_{\xi}({\bf z}^{t}|{\bf x}^{t})} \left[ \log \frac{p_{\theta}({\bf x}^{t},{\bf z}^{t})}{q_{\xi}({\bf z}^{t}|{\bf x}^{t })} \right]  + \mathbb{E}_{q_{\xi}({\bf z}^{t-1}|\tilde{\bf x}^{t-1})}\left[ \log \frac{p_{\theta}(\tilde{\bf x}^{t-1},
{\bf z}^{t-1})}{q_{\xi}({\bf z}^{t-1}| \tilde{\bf x}^{t-1})} \right]
\end{aligned}
\end{equation}
We define the above equation as ${\mathcal L}(\theta ,\xi ;{\bf x}^{t},\tilde{\bf x}^{t-1})$. 
\end{proof}

\begin{lemma}
From the Theorems 2 and 3, we can derive a lower bound on the sample log-likelihood at t-th task learning, as expressed by:

\begin{equation}
\begin{aligned}
\log {p_\theta }({{\bf{x}}^1},..,{{\bf{x}}^t}) &\ge
\mathcal{L}(\theta ,\xi ;{{\bf{x}}^1},..,{{\bf{x}}^t}) \ge \mathcal{L}(\theta ,\xi ;{{\bf{x}}^{t },{\tilde {\bf{x}}^{t-1}})} \\&- {\rm{\bf W}(}v{\rm{,}}v'{\rm{)}} - \sqrt {2\log \left( {\frac{1}{\delta }} \right)/a'} \left( {\sqrt {\frac{1}{n}}  + \sqrt {\frac{1}{{n'}}} } \right) - D^*
\end{aligned}
\end{equation}
where $\nu  \in {\bf R}^s,\nu' \in {\bf R}^s$ are formed by
${n}$ and ${n'}$ numbers of samples drawn from $p({{\bf{x}}^t})p({{{\bf{\tilde x}}}^{t-1}})$ and $\prod\nolimits_i^t {p({{\bf{x}}^i})} $, respectively, where $n$ and $n'$ denote the sample size. 
\end{lemma}

Lemma 2 provides an explicit way to investigate how information is lost through GRMs during the lifelong learning process. Meanwhile, Lemma 3 derives the evidence lower bound (ELBO) of the sample log-likelihood. All GRM approaches, based on VAE and GAN architectures, can be explained through this theoretical analysis. However, GAN based approaches lack inference mechanisms, such as the one provided in Lemma 3. This motivates us to develop a new lifelong learning approach utilizing the advantages of both GANs and VAEs, enabling the generation of data and the log-likelihood estimation abilities.

\vspace{-5pt}
\section{The two-step latent variables optimization over time}
\label{GenReplay}
\vspace{-5pt}

In the following, inspired by the theoretical analysis from Section~\ref{TheoGenReplay}, we introduce a two-step optimization algorithm which besides training a powerful generative replay network it also induces latent representations. Our algorithm is different from existing hybrid models which would only train generator and inference models within a single optimization function \cite{HybridGAN} or learn an optimal coupling between generator and inference model by using adversarial learning \cite{SymmetricVAE,AFL,AdLearnInf,HybridGAN,Alice,AAE,AdbVB,Veegan}. The proposed algorithm contains two independent optimization paths, namely ``wake''  and ``dreaming'' phases. In the following we explain the application of the proposed Lifelong VAEGAN in supervised, semi-supervised and unsupervised learning.

\vspace{-10pt}
\subsection{Supervised learning}

In the ``wake'' phase, by considering the Definitions 1 and 2, the refined distribution $p(\tilde{\bf x}^t)$ is trained to approximate $p(\tilde{\bf x}^{t-1},{\bf x}^t)$ by minimizing the Wasserstein distance~: 
\begin{equation}
\begin{aligned}
& \mathop {\min }\limits_G \mathop {\max }\limits_D {\mathcal L}_{GAN}^G (\theta_t ,\omega_t) \buildrel \Delta \over = \mathbb{E}_{p ({\bf z}),
p ( {\bf c}),p ({\bf a})} 
[ D ( G ( {\bf c},{\bf z},{\bf a})) ] - \mathbb{E}_{p( \tilde{\bf x}^{t-1})
p({\bf x}^t)} [D({\bf x}) ],
\end{aligned}
\end{equation}
where we omit the penalty term, weighted by $\lambda$, for the sake of simplification.

In the ``dreaming'' phase, we maximize the sample log-likelihood on the joint distribution of the generated data and the empirical data, associated with a given new task, by maximizing the ELBO~:
\begin{equation}
\begin{aligned}
\mathcal {L}_{VAE}(\theta_t ,\varsigma_t ,\varepsilon_t ,\delta_t) &\buildrel \Delta \over = \mathbb{E}_{q_{\varsigma,\varepsilon,\delta}({\bf z},{\bf a},{\bf c}|{\bf x}^t)} 
\left[ \log \frac{p_\theta({\bf x}^t| {\bf z},{\bf a},{\bf c})}{q_{\varsigma ,\varepsilon ,\delta } ({\bf z},{\bf a},{\bf c}|{\bf x}^t)} \right] \\
& + \mathbb{E}_{q_{\varsigma ,\varepsilon ,\delta }({\bf z},{\bf a},{\bf c}|\tilde{\bf x}^{t - 1})} 
\left[ \log \frac{p_\theta ( \tilde{\bf x}^{t-1}|{\bf z},{\bf a},{\bf c})}{q_{\varsigma ,\varepsilon ,\delta }({\bf z},{\bf a},{\bf c}|\tilde{\bf x}^{t-1})} \right].
\end{aligned}
\end{equation}

This loss function is used to train both Generator and Inference models. After training, the Inference model $q_\delta ({\bf c}|{\bf x})$ can be used for classification.

\subsection{Semi-supervised learning}
\label{Semisuperv}

Let us consider that we only have a small subset of labelled data from each database, while the rest of data is unlabelled. For the labelled data, we derive the objective function without the inference model $q_\delta ({\bf c}|{\bf x})$ in the ``dreaming'' phase as~:
\begin{equation}
\begin{aligned}
& \mathcal {L}_{VAE}^S (\theta_t ,\varsigma_t ,\varepsilon_t ,\delta_t) \buildrel \Delta \over = \ \sum\nolimits_{}^2 \mathbb{E}_{q_\varsigma ( {\bf z}|{\bf x}), q_\varepsilon ({\bf a}|{\bf x}),p({\bf y})} [ \log p_\theta ( {\bf x}|{\bf z},{\bf a},{\bf y} ) ]\\
& - {D_{KL}}[{q_\varsigma }({\bf{z}}|{\bf{x}})||p({\bf{z}})] - {\mathbb{E}_{{q_\varsigma }({\bf{z}}|{\bf{x}})}}{D_{KL}}[{q_\varepsilon }({\bf{a}}|{\bf{z}})||p({\bf{a}}|{\bf{z}})] - {D_{KL}}[{q_\delta }({\bf{c}}|{\bf{x}})||p({\bf{c}})],
\end{aligned}
\end{equation}
where $\sum^2$ denotes the estimation of ELBO on the joint model log-likelihood. In addition, we model the unlabeled data samples by using 
$\mathcal L_{VAE} ( \theta_t ,\varsigma_t ,\varepsilon_t ,\delta_t )$, where the discrete variable ${\bf c}$ is sampled from the Gumbel-softmax distribution whose probability vector is obtained by the encoder $q_\delta ( {\bf c}|{\bf x})$. The full semi-supervised loss used to train the hybrid model is defined as:
\begin{equation}
\begin{aligned}
{\mathcal L}_{VAE}^{Semi} \buildrel \Delta \over = {\mathcal L}_{VAE}^S + \beta {\mathcal L}_{VAE}^{},
\end{aligned}
\end{equation}
where $\beta$ is used to control the importance of the unsupervised learning when compared with the component associated with supervised learning. In addition, the entropy loss ${\mathcal L}_c ( \delta )$ is also performed with the labeled data samples in order to enhance the prediction ability of $q_\delta ( {\bf c}|{\bf x})$.

\subsection{Unsupervised learning}
\label{describe_unsupervised}

We also employ the proposed hybrid VAE-GAN model for the lifelong unsupervised learning setting, where we do not have the class labels for each task. Similarly to the supervised learning framework, we minimize the Wasserstein distance between $p(\tilde{\bf x}^t)$ and $p(\tilde{\bf x}^{t-1},{\bf x}^t)$~:
\begin{equation}
\begin{aligned}
& \mathop {\min }\limits_G \mathop {\max }\limits_D {\mathcal L}_{GAN}^U (\theta _t ,\omega _t) \buildrel \Delta \over = \mathbb{E}_{p ({\bf z}),p ( {\bf a} )} [ D (G ({\bf z},{\bf a}))] - \mathbb{E}_{p( \tilde{\bf x}^{t-1}) p({\bf x}^t)} [D({\bf x})] .
\end{aligned}
\label{L_GAN_U}
\end{equation}

In the ``dreaming'' phase, we train the generator and inference models using:
\begin{equation}
\begin{aligned}
{\mathcal L}_{VAE}^U ( \theta_t ,\varsigma_t ,\varepsilon_t  ) \buildrel \Delta \over = &
\sum\nolimits_{}^2 \mathbb{E}_{ q_\varsigma ( {\bf z}|{\bf x}), q_\varepsilon ({\bf a}|{\bf x})} [ \log p_\theta ({\bf x}|{\bf z},{\bf a} ) ]\\
&- {D_{KL}}[{q_\varsigma }({\bf{z}}|{\bf{x}})||p({\bf{z}})] - {\mathbb{E}_{{q_\varsigma }({\bf{z}}|{\bf{x}})}}{D_{KL}}[{q_\varepsilon }({\bf{a}}|{\bf{z}})||p({\bf{a}}|{\bf{z}})],
\end{aligned}
\label{L_VAE_U}
\end{equation}
where the Generator is conditioned on only two variables, ${\bf z}$ and ${\bf a}$. For learning disentangled representations, we employ the Minimum Description Length (MDL) principle \cite{UnVAE,MDLVAEL} and replace the second term from equation (\ref{L_VAE_U}) by $\gamma |D_{KL} [ q_\varsigma ({\bf z}|{\bf x})||p ( {\bf z}) ] - C|$,  where $\gamma$ and $C$ are a multiplicative and a linear constant used for controling the degree of disentanglement. 

\vspace{-5pt}
\section{Experimental results}
\label{Exp}

\subsection{Lifelong unsupervised learning}
\label{UnsupervisedLearning_section}
In this section, we investigate how the proposed Lifelong VAEGAN (L-VAEGAN) model learns meaningful and interpretable representations in images from various databases under the unsupervised lifelong learning setting. 

\noindent \textbf{Reconstruction and Interpolation results following lilfeong learning.} We train the L-VAEGAN model using the loss functions  ${\mathcal L}_{GAN}^U$ and $\mathcal{L}_{VAE}^U$  from equations (\ref{L_GAN_U}) and (\ref{L_VAE_U}), which contain  adversarial and VAE learning terms, respectively, and we consider a learning rate of 0.001. The generation and reconstruction results are presented in Figs.~\ref{ReconInterp}a-c, for the lifelong learning of CelebA \cite{Celeba} to CACD \cite{CACD}, and in Figs.~\ref{ReconInterp}d-f for CelebA to 3DChair \cite{ThreeDChairs}. From these results, the inference model works well for both discrete and continuous latent variables. 

\begin{figure}[h]
	\centering
		\setlength{\abovecaptionskip}{0pt}   
    \setlength{\belowcaptionskip}{-15pt}
	\subfigure[Real Images.]{
		\centering
		\includegraphics[scale=0.33]{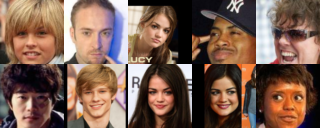}
	}
	\subfigure[Generated Images.]{
		\centering
		\includegraphics[scale=0.33]{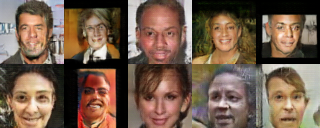}
	}
	\subfigure[Reconstructions.]{
		\centering
		\includegraphics[scale=0.33]{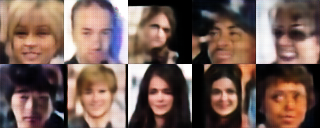}
	} \\
	\subfigure[Real Images.]{
		\centering
		\includegraphics[scale=0.33]{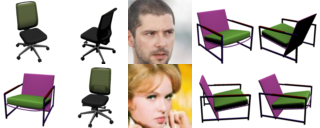}
	}
	\subfigure[Generated Images.]{
		\centering
		\includegraphics[scale=0.33]{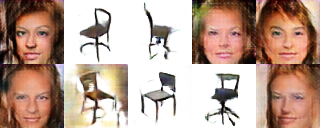}
	}
	\subfigure[Reconstructions.]{
		\centering
		\includegraphics[scale=0.33]{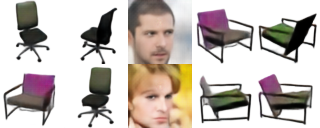}
	}
	\caption{Image reconstructions and generated after the CelebA to CACD (top row) and CelebA to 3DChair (bottom row) lifelong learning.}
	\label{ReconInterp}
\end{figure}

In the following we perform data interpolation experiments under the lifelong learning setting in order to evaluate the manifold continuity. We call lifelong interpolation when the interpolation is performed between multiple domains, by considering data from different databases, under the lifelong learning setting. We randomly select two images and then infer their discrete ${\bf a}$ and continuous ${\bf z}$ latent variables by using the inference model. The interpolation results are shown in Figs.~\ref{Intertable}-a and \ref{Intertable}-b, for CelebA to CACD and CelebA to 3D-chair lifelong learning, respectively. The first two rows show the interpolation results in images from the same database, while the last two rows show the interpolations of images from two different databases. From the images from the last two rows of Fig.~\ref{Intertable}-b we observe that a chair is smoothly transformed into a human face, where its seat and backside are smoothly changed into the eyes and hair of a person. This shows that the L-VAEGAN model can learn the joint latent space of two completely different data configurations.  

\noindent \textbf{Lifelong Disentangled Representations.} We train the L-VAEGAN model under the CelebA to 3D-Chairs lifelong learning by adapting the loss functions from (\ref{L_GAN_U}) and (\ref{L_VAE_U}) in order to achieve unsupervised disentangled representations, as mentioned in Section~\ref{describe_unsupervised}. We consider the multiplicative parameter $\gamma=4$, while increasing the linear one $C$ from 0.5 to 25.0 during the training. After the training, we change one dimension of a continuous latent representation ${\bf z}$, inferred by using the inference model for a given input and then recover it back in the visual data space by using the generator. The disentangled results are presented in Fig~\ref{UnsupervisedDisentangled} which indicates changes in the skin, gender, narrowing of the face, chair size, face pose and chairs' style. These results show that the L-VAEGAN hybrid model can discover different disentangled representations in both CelebA and 3D chair databases.

\begin{figure}[h]
	\centering
		\setlength{\abovecaptionskip}{0pt}   
    \setlength{\belowcaptionskip}{-12pt}  
	\subfigure[CelebA to CACD learning.]{
		\centering
		\includegraphics[scale=0.425]{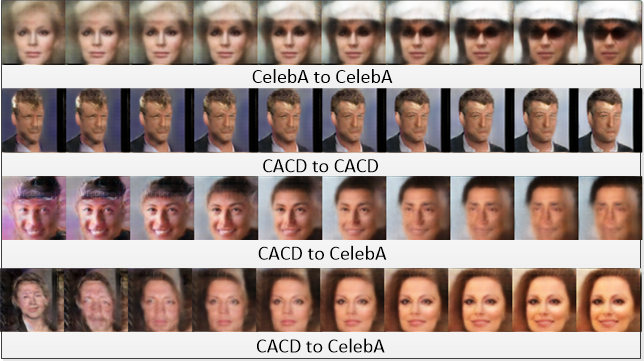}
	}
	\subfigure[CelebA to 3D-chair learning.]{
		\centering
		\includegraphics[scale=0.425]{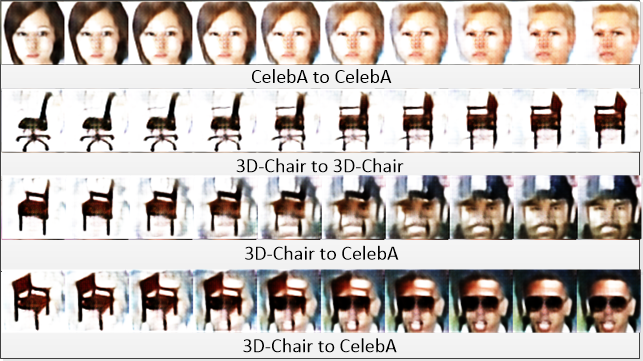}
	}
\centering
	\caption{Interpolation results after lifelong learning. }
	\label{Intertable}
\end{figure}

\begin{figure*}[htbp]
	\centering
		\setlength{\abovecaptionskip}{0pt}   
    \setlength{\belowcaptionskip}{0pt}  
	\includegraphics[scale=0.267]{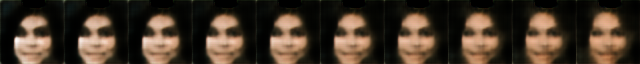}
		\includegraphics[scale=0.267]{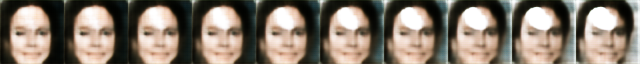}
	\includegraphics[scale=0.267]{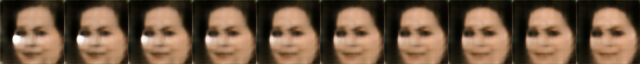}
		\includegraphics[scale=0.267]{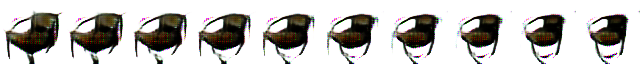}
	\includegraphics[scale=0.267]{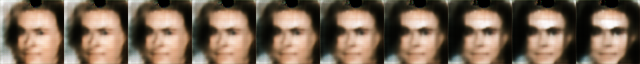}
	\includegraphics[scale=0.267]{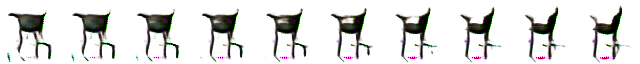}
	\centering
	\caption{Manipulating latent variables on CelebA and 3D-chair datasets, under the CelebA to 3D-chairs lifelong learning. We change a single latent variable from -3.0 to 3.0 while fixing all the others. From left to right and top to bottom, we can see changes in skin, gender, narrowing of the face, chair size, face pose and the style of chairs.    }
	\label{UnsupervisedDisentangled}
\end{figure*}

\makeatletter\def\@captype{table}\makeatother
\begin{minipage}{.5\textwidth}	
\caption{Quantitative evaluation of the representation learning ability.}
\begin{tabular}{lcccc}
		\toprule
		\multicolumn{5}{c}{The lifelong learning of MNIST and Fashion}                   \\
		\cmidrule(r){1-5}
		Methods   & Lifelong &Dataset &Rec&Acc \\
		\midrule
		L-VAEGAN & M-F &MNIST & \textBF{4.75}& \textBF{92.53} \\
		LGM \cite{GenerativeLifelong} & M-F &MNIST & 7.18& 91.26 \\
		VAEGAN \cite{AdbVB} & M-F &MNIST & 6.54& 91.87\\
		
				\cmidrule(r){1-5}

		L-VAEGAN & M-F &Fashion & \textBF{17.44}& \textBF{67.66} \\
		LGM \cite{GenerativeLifelong} & M-F &Fashion & 18.33& 66.17 \\
		VAEGAN \cite{AdbVB} & M-F &Fashion & 17.03& 67.23
		\\
				\cmidrule(r){1-5}

			L-VAEGAN & F-M &MNIST & \textBF{4.92}& \textBF{93.29} \\
		LGM \cite{GenerativeLifelong} & F-M &MNIST & 7.18& 91.26 \\
		VAEGAN \cite{AdbVB} & F-M &MNIST & 5.52& 92.16\\
				\cmidrule(r){1-5}
		
				L-VAEGAN & F-M &Fashion & \textBF{13.16}& \textBF{66.97} \\
		LGM \cite{GenerativeLifelong} & F-M &Fashion & 18.83& 62.53 \\
		VAEGAN \cite{AdbVB} & F-M &Fashion & 16.57& 64.98\\
		\bottomrule
	\end{tabular}
	\label{tab10}
\end{minipage} \;\;\;\;\;\;\;\;\;\;
\makeatletter\def\@captype{figure}\makeatother
\begin{minipage}{.40\textwidth}
\centering
\subfigure[IS evaluation.]{
		\centering
		\includegraphics[scale=0.28]{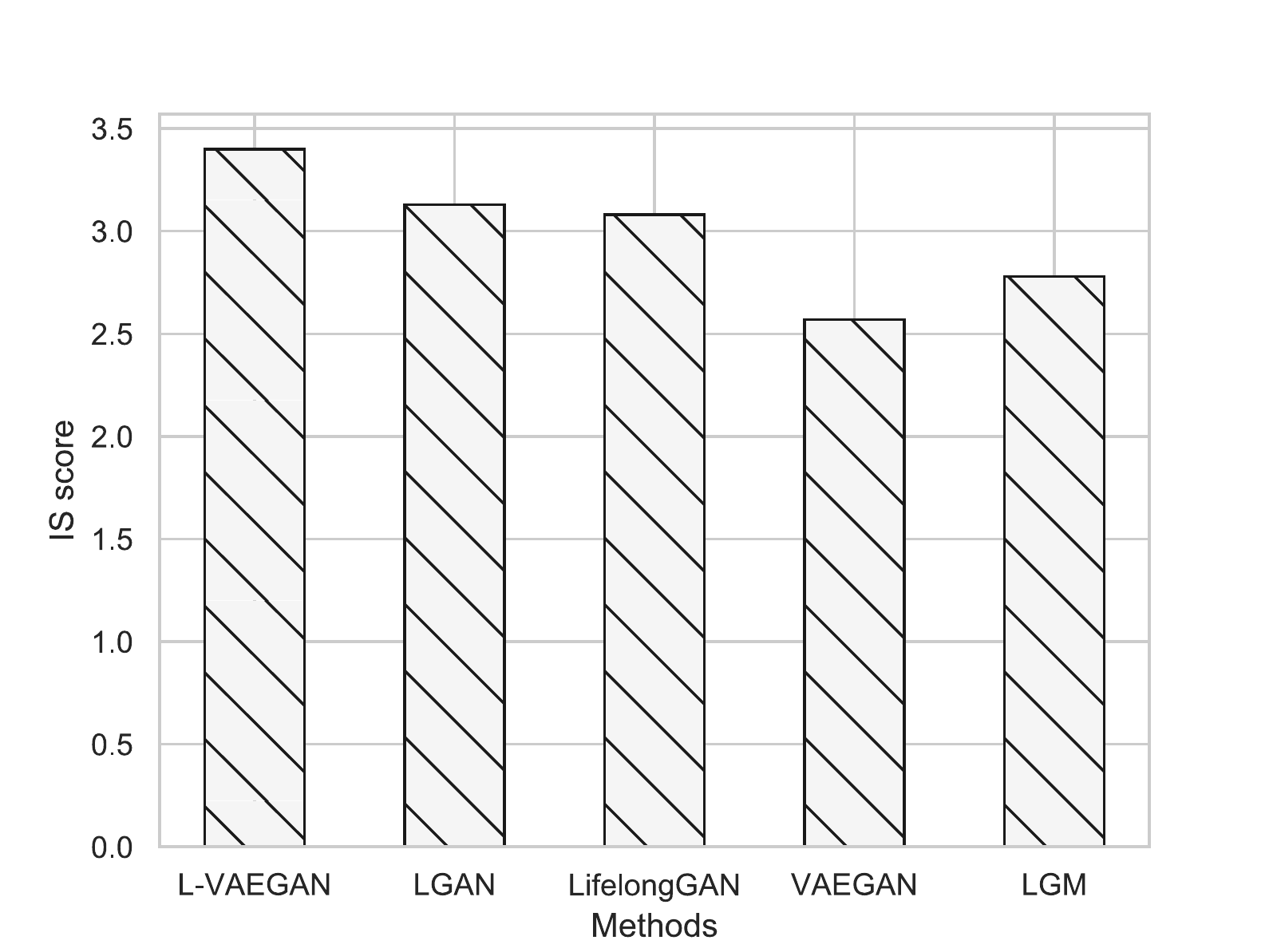}
	}
	\subfigure[FID evaluation.]{
		\centering
		\includegraphics[scale=0.28]{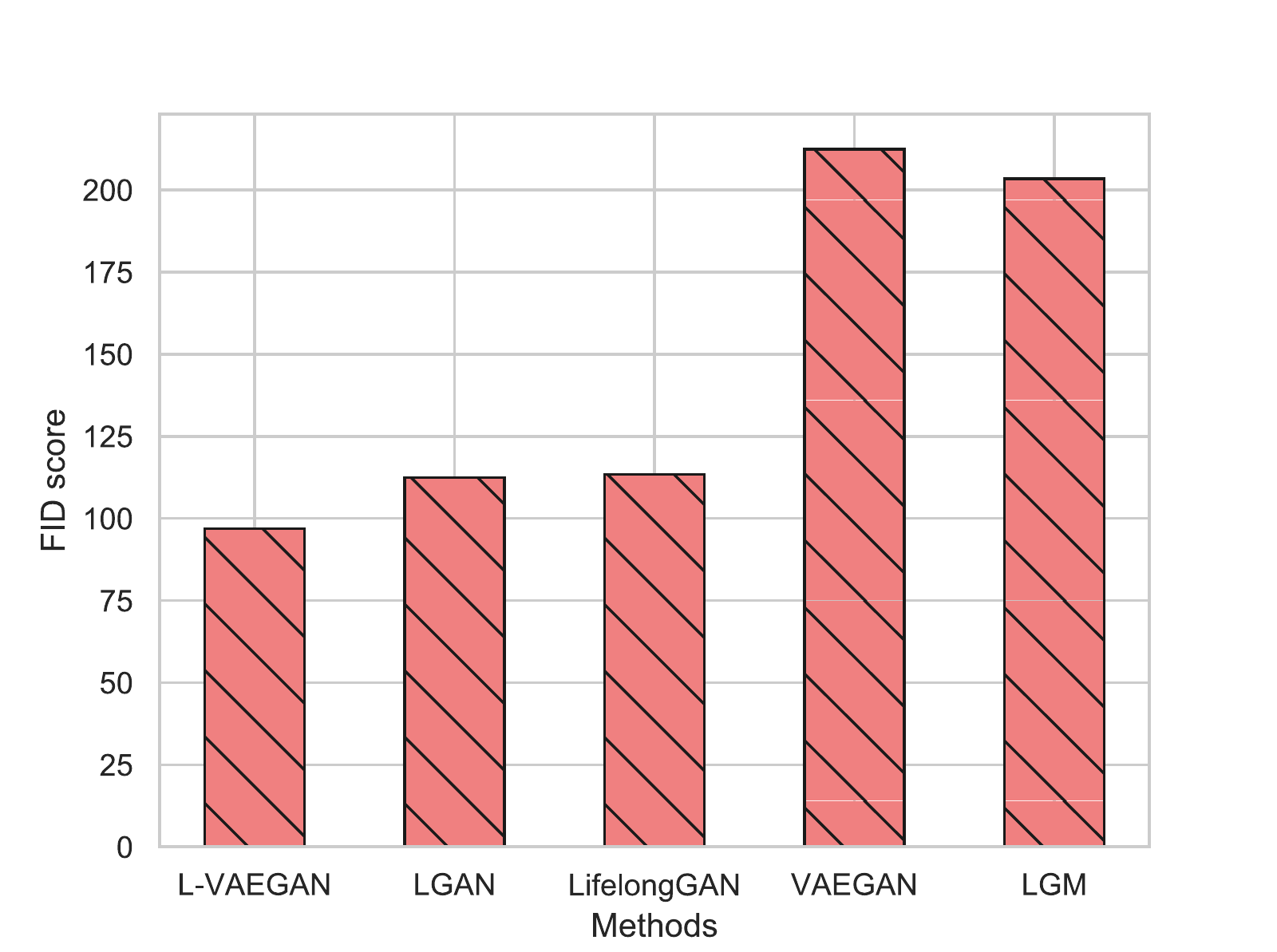}
	}
	\caption{Quality evaluation.}
	\label{inception}
\end{minipage}

\noindent \textbf{Quantitative Evaluation.} In the following we evaluate numerically the image representation ability for the proposed approach. We consider the reconstruction (Rec) as MSE, and classification accuracy (Acc), as in \cite{LGAN_conditional}. The classification accuracy is calculated by a classifier which was trained on the generated images and evaluated on unseen testing samples. We provide the results in Table~\ref{tab10}, where L-VAEGAN learns firstly the database MNIST followed by Fashion, denoted as
M-F, while F-M represents the learning of the same databases in reversed order. For comparison we consider LGM \cite{GenerativeLifelong} and VAEGAN \cite{AdbVB}, which is one of the best known hybrid models enabled with an inference mechanism. We implement VAEGAN using GRM in order to prevent forgetting. From Table~\ref{tab10} we can see that L-VAEGAN achieves the best results. 

We also use the Inception score (IS) \cite{InceptScore} and Fréchet Inception Distance (FID)  \cite{FID} for evaluating the quality of the generated images. We consider using the IS score for evaluating the generated images when training Cifar10 \cite{CIFAR10} to MNIST lifelong learning and the FID score for CelebA to CACD lifelong learning. Figs.~\ref{inception}-a and \ref{inception}-b show the IS and FID resuls, respectively, where a lower FID and a higher IS indicate better quality images and where we compare with four lifelong learning approaches~: LGAN \cite{Generative_replay}, LifelongGAN \cite{LGAN_conditional}, VAEGAN \cite{AdbVB} and LGM \cite{GenerativeLifelong}. LifelongGAN \cite{LGAN_conditional} requires to load the previously learnt real data samples in order to prevent forgetting them when it is applied in general generation tasks (no conditional images are available). The results indicate that GAN-based lifelong approaches achieve better scores than VAE based methods because VAEs usually generate blurred images. The proposed approach not only produces higher-quality generative replay images but also learns representations of data that other GAN-based lifelong approaches can not. When compared with the VAE-based methods, the proposed approach yields higher-quality reconstructed and generated images. 

\subsection{Lifelong supervised learning}

We compare L-VAEGAN with various other methods under the lifelong supervised setting. LGAN \cite{Generative_replay} typically trains a classifier (called Solver) on both the images generated by the GAN and the training data samples from the current task. We also consider an auxiliary classifier for LGM \cite{GenerativeLifelong} (which cannot be used directly as a classifier) by training it on the mixed dataset consisting of images generated by the LGM and the training samples of the current task.

\begin{table}[h]
	\centering
			\caption{The classification results for the lifelong learning of MNIST and Fashion.}
			\setlength{\abovecaptionskip}{-20pt}   
    \setlength{\belowcaptionskip}{0pt}
			\small
	\begin{tabular}{lccccccc}
		\toprule
		Dataset   & Lifelong  &L-VAEGAN    &LGAN\cite{Generative_replay} &LGM\cite{GenerativeLifelong}&EWC \cite{EWC} &Transfer&MeRGANs \cite{MemoryReplayGAN} \\
		\midrule
		MNIST & M-F &\textBF{98.76} & 98.41& 97.29&37.7 &40.63&98.34 \\
		MNIST & F-M &98.77 & 98.32& 98.85&\textBF{99.12}&98.25&98.27 \\
		Fashion & M-F &\textBF{92.01} & 91.42& 91.71&91.38&91.01&91.12 \\
		Fashion & F-M &\textBF{89.24} & 89.15& 86.05&54.53&37.92&88.86 \\
		\bottomrule
	\end{tabular}
	\label{tab1}
\end{table}

\begin{table}[h]
	\centering
			\caption{The classification results for the lifelong learning of MNIST and SVHN.}
		\setlength{\abovecaptionskip}{-20pt}   
    \setlength{\belowcaptionskip}{0pt}
	\small
	\begin{tabular}{lccccccc}
		\toprule
		Dataset   & Lifelong  &L-VAEGAN    &LGAN\cite{Generative_replay} &LGM\cite{GenerativeLifelong}&EWC \cite{EWC}&Transfer&MeRGANs \cite{MemoryReplayGAN} \\
		\midrule
		MNIST & M-S &\textBF{96.63} & 96.59& 96.59&32.63&4.70&96.46 \\
		MNIST & S-M &98.55 & 98.42& 98.30&\textBF{99.03}&98.45&98.38 \\
		SVHN & M-S &81.97 & 80.77& 80.10&89.71&\textBF{89.90}&79.67 \\
		SVHN & S-M &\textBF{80.99} & 76.76& 80.97&37.69&3.60&80.85 \\
		\bottomrule
	\end{tabular}
	\label{tab2}
\end{table}

We train the L-VAEGAN model under the MNIST to SVHN \cite{SVHN} and MNIST to Fashion lifelong learning tasks, respectively. The classification accuracy of various methods is reported in Tables~\ref{tab1} and \ref{tab2}. We observe that the replay generative images can prevent forgetting and the quality of performing on the previous tasks is depending on the capability of the Generator. For instance, the classification accuracy of previous tasks for LGM is a bit lower than for the other GAN-based methods. The reason for this is that LGM uses VAEs as a generator and therefore can not generate high-quality images when compared to GANs. The classification results demonstrate that the proposed approach can be used in the lifelong supervised learning with a good performance.

\makeatletter\def\@captype{table}\makeatother
{
\centering
\begin{minipage}{.55\textwidth}
		\setlength{\abovecaptionskip}{12pt}   
    \setlength{\belowcaptionskip}{1pt}
		\caption{Semi-supervised classification error results on MNIST database, under the MNIST to Fashion lifelong learning.}
\begin{tabular}{lcc}
		\toprule
		Methods    & Lifelong? & Error  \\
		\midrule
		L-VAEGAN & Yes &4.34 \\
		LGAN \cite{Generative_replay} & Yes &5.46 \\
		Neural networks (NN) \cite{SemiSuper} & No &10.7 \\
		(CNN) \cite{SemiSuper} & No &6.45 \\
		TSVM \cite{SemiSuper} & No &5.38 \\
		CAE \cite{SemiSuper} & No &4.77 \\
		M1+TSVM \cite{SemiSuper} & No &4.24 \\
		M2 \cite{SemiSuper} & No &3.60 \\
		M1+M2 \cite{SemiSuper} & No &2.40 \\
		Semi-VAE \cite{Semi_VAE} & No &2.88 \\
		\bottomrule
	\end{tabular}
	\label{tab3}
\end{minipage}
\makeatletter\def\@captype{figure}\makeatother
\begin{minipage}{.43\textwidth}
\centering
		\includegraphics[scale=0.36]{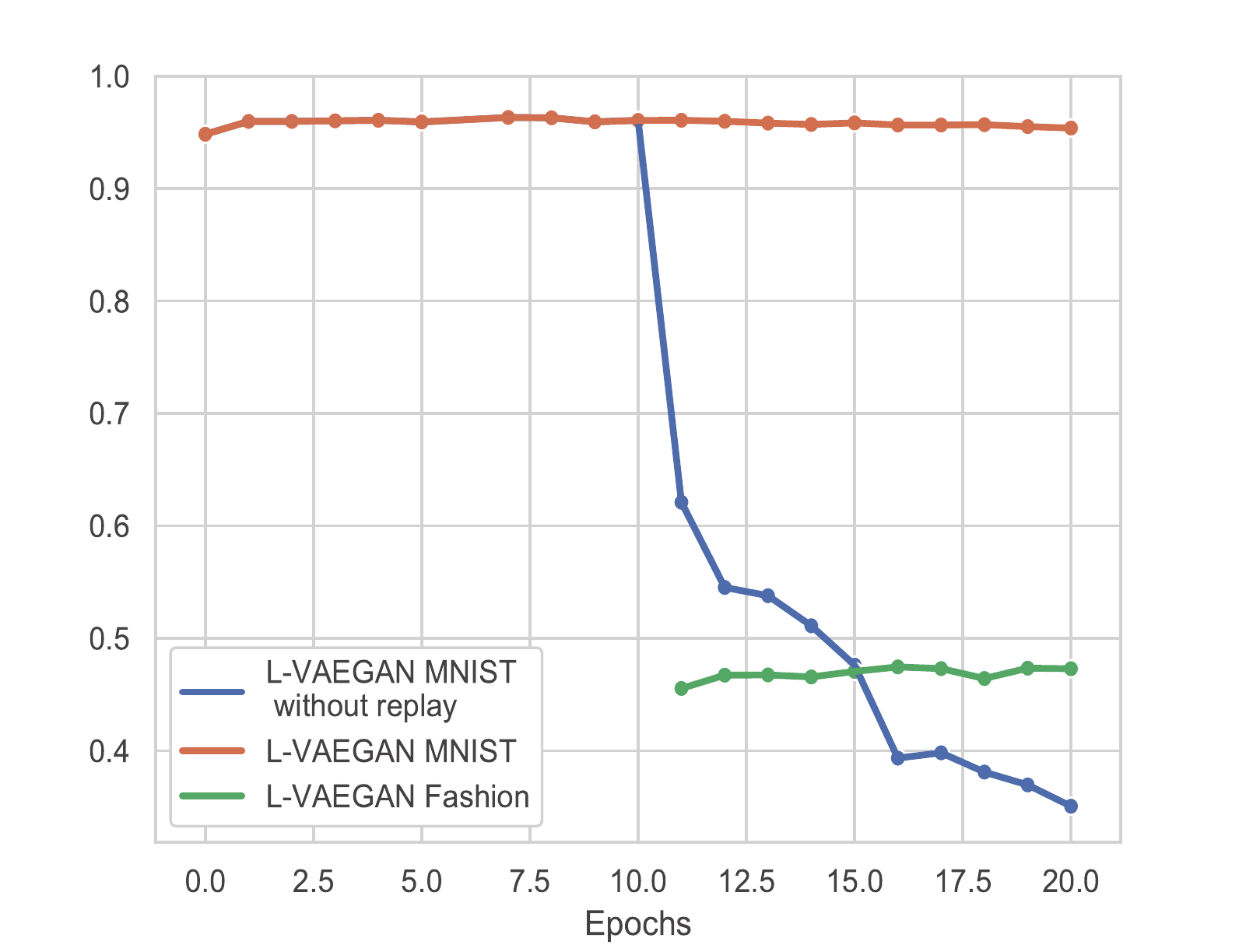}
	\centering
	\caption{The accuracy during the semi-supervised training on the testing data samples during the lifelong learning. The model is trained for 10 epochs for each task. }
	\label{Semi_accuracy}
\end{minipage}
}

\subsection{Lifelong semi-supervised learning}

 The semi-supervised learning framework for L-VAEGAN is presented in Section~\ref{Semisuperv}. During the semi-supervised training we consider only a small number of labelled images from each database.  In the following experiments we divide MNIST and Fashion datasets into two subsets each, representing labelled and unlabeled data. We consider a total of 1,000 and 10,000 labelled images for MNIST and Fashion datasets, respectively, with an equal number of data in each class for the labelled set. We train the proposed L-VAEGAN model with 10 epochs for learning each database and the resulting classification plots are shown in Fig.~\ref{Semi_accuracy}. We observe that without generative replay samples, the model under the semi-supervised setting suffers from catastrophic forgetting.

The classification results for lifelong learning when using L-VAEGAN compared to other semi-supervised learning methods are provided in Table~\ref{tab3}. These results show that the proposed approach outperforms LGAN \cite{Generative_replay}, under the semi-supervised learning setting, and achieves competitive results when compared to the state-of-the art models which are not trained using lifelong learning.

\subsection{Analysis}

The plot of the classification accuracy for all testing samples during the MNIST to Fashion lifelong learning is provided in Fig.~\ref{CAccuracy}-a, where the first 10 training epochs correspond to learning the MNIST task, while the next 10 epochs are used for learning the Fashion database. We observe that the performance of previous tasks is maintained when considering the GRM. However, without replaying data, the model learned from the previous tasks quickly forgets the past knowledge while learning new tasks, as indicated by a significant performance drop, as observed in Fig.~\ref{CAccuracy}-a. Besides, unlike LGAN \cite{Generative_replay}, the proposed L-VAEGAN is able to provide inference mechanisms, benefiting many downstream tasks, as shown in the results from Fig.~\ref{ReconInterp}. 

In the following, in Fig.~\ref{CAccuracy}-b we provide numerical results for the generalization bounds for the GRM, described in Section~\ref{TheoGenReplay},  where risk1 and risk2 denote $E(h^{\rm{1}}(\mu_1))$ and $E(h^{\rm{1}}(\mu_{1'}))$, respectively. We find that $E(h^{\rm{1}}(\mu_1))$ is very close to $E(h^{\rm{1}}(\mu_{1'}))$ and still a bound on $E({h^{\rm{1}}}(\mu_{1'}))$ during the training. Fig.~\ref{CAccuracy}-c provides the numerical results for the observed risks under the MNIST-MNIST lifelong learning. We observe that $E({h^{\rm{1}}}(\mu_1))$ is still a bound on $E({h^{\rm{1}}}(\mu_{1'}))$ and this bound is gradually increased during the training. The reason for this is that the model is gradually adapting $p({\bf{\tilde x}}^1)$ and the bound depends on the distance between $p({\bf{\tilde x}}^1)$ and $p({\bf x}^1)$.

\begin{figure}[htbp]
	\centering
		\setlength{\abovecaptionskip}{0pt}   
    \setlength{\belowcaptionskip}{-1cm}  
	\subfigure[Forgetting curves]{
		\centering
		\includegraphics[scale=0.23]{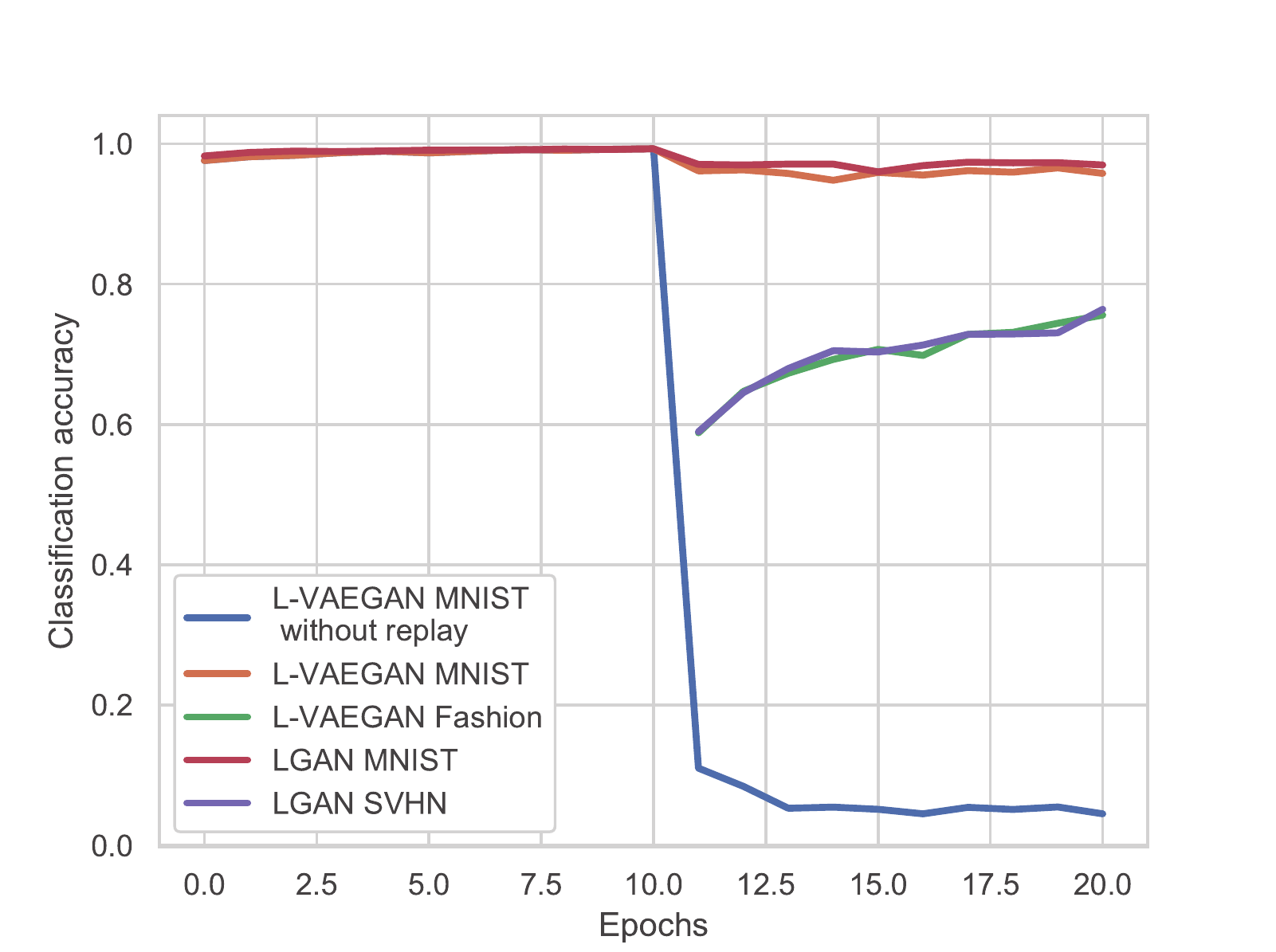}
	}
		\subfigure[MNIST-Fashion ]{
		\centering
		\includegraphics[scale=0.23]{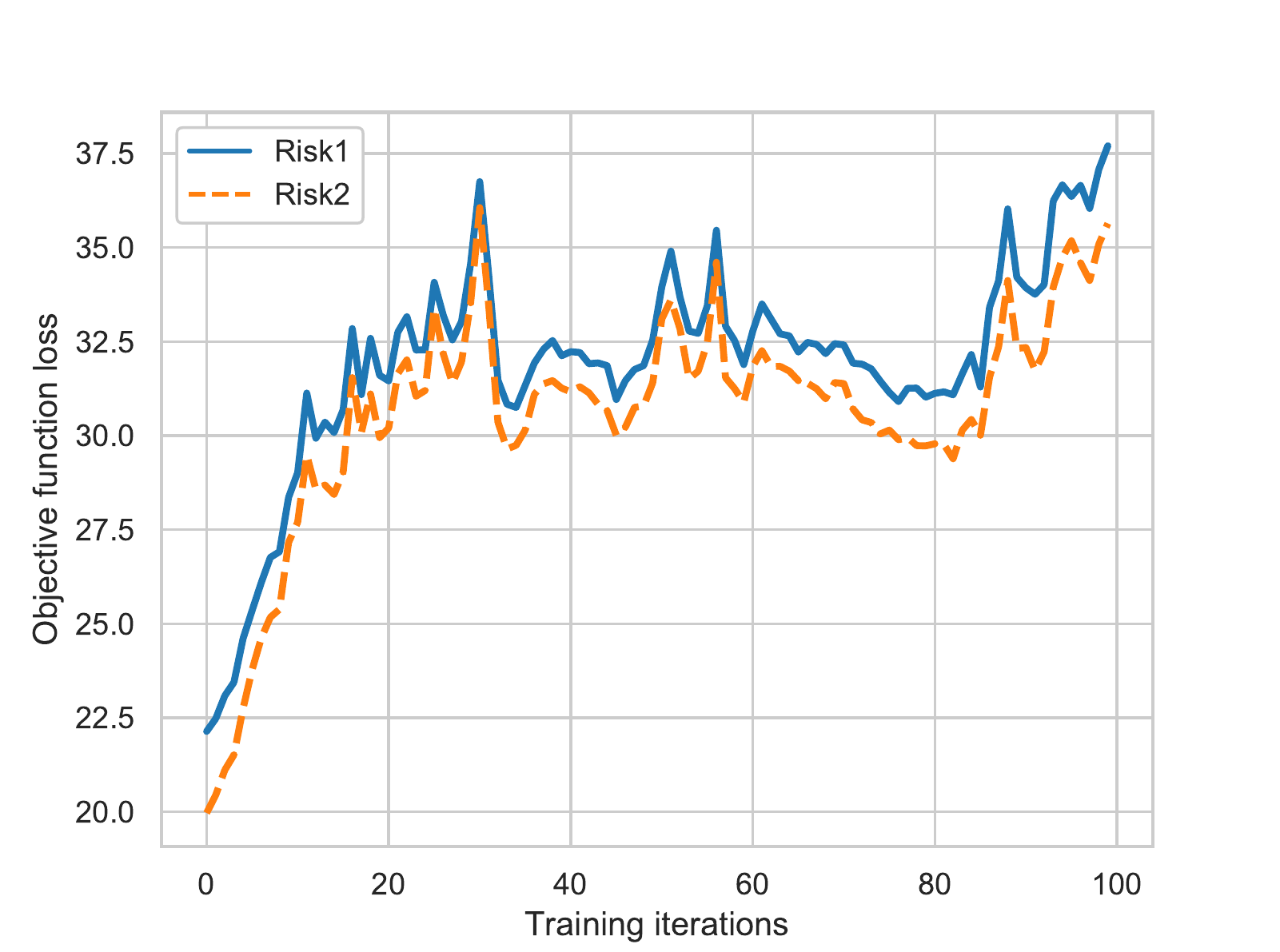}
	}
		\subfigure[MNIST-MNIST]{
		\centering
		\includegraphics[scale=0.23]{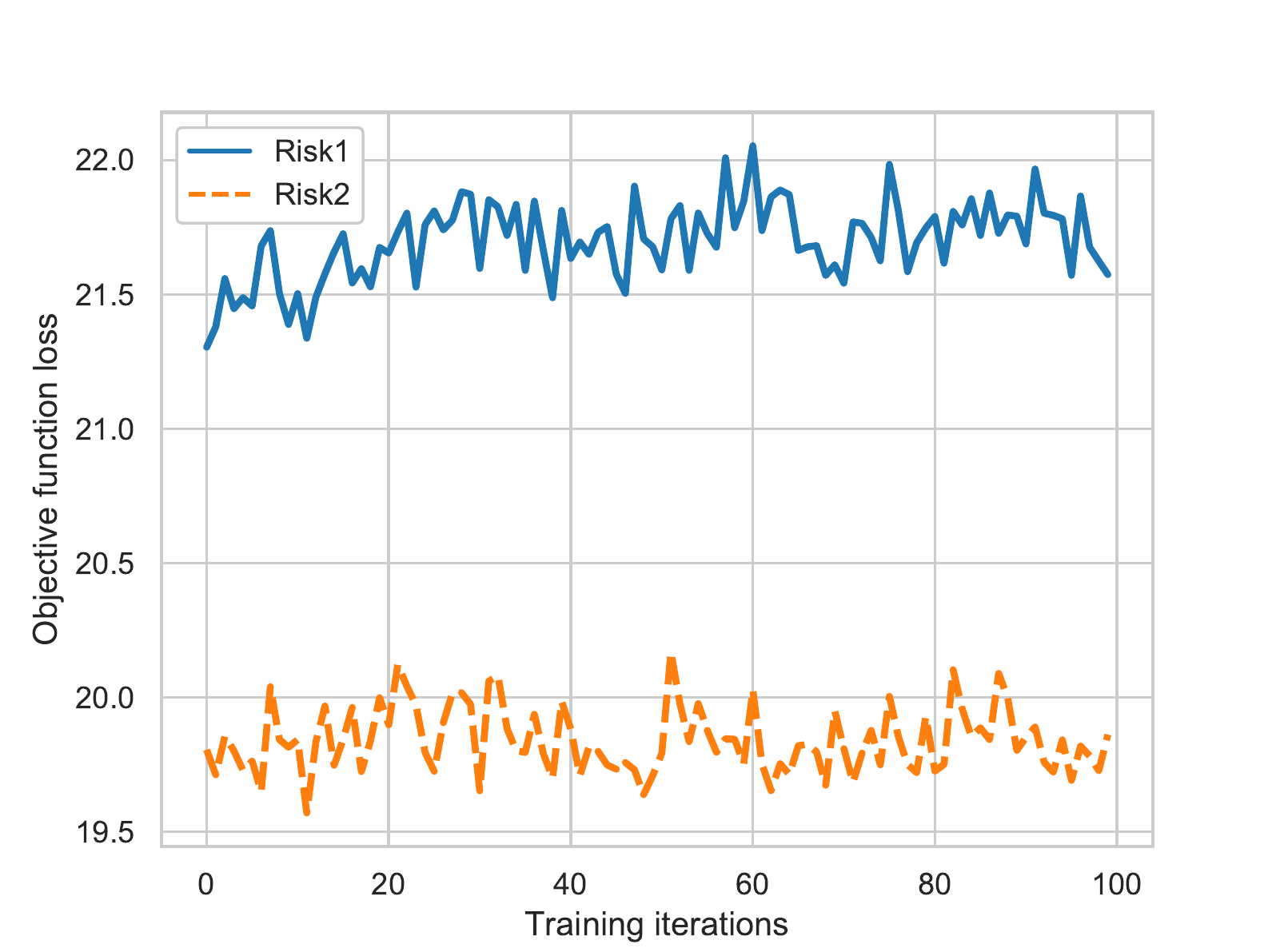}
	}
	\centering
	\caption{Classification accuracy and observed risks during the lifelong learning.}
	\label{CAccuracy}
\end{figure}

\section{Conclusion}
\label{Con}

A novel hybrid model for lifelong representation learning, called Lifelong VAEGAN (L-VAEGAN) is proposed in this research study. A two-step optimization algorithm, which can also induce higher-quality generative replay samples and learn informative latent representations over time, is used to train L-VAEGAN. The results indicate that L-VAEGAN model is able to discover disentangled representations from multiple domains under the lifelong learning setting. More importantly, L-VAEGAN automatically learns semantic meaningful shared latent variables across different domains, which allows to perform cross-domain interpolation and inference.

\clearpage
%
%

\end{document}